\documentclass[12pt]{elsarticle}
\journal{}
\makeatletter
\def\ps@pprintTitle{%
 \let\@oddhead\@empty
 \let\@evenhead\@empty
 \let\@oddfoot\@empty
 \let\@evenfoot\@empty
}
\makeatother

\usepackage[utf8]{inputenc}

\usepackage{amssymb,amsmath}
\newtheorem{definition}{Definition}
\newtheorem{theorem}{Theorem}
\newtheorem{lemma}{Lemma}
\newtheorem{claim}{Claim}[lemma]

\newenvironment{proof}{\begin{trivlist}\item\noindent{\sc Proof.}}{\hfill$\Box\hspace{2mm}$\end{trivlist}}

\newenvironment{proof-of-claim}{\begin{trivlist}\item\noindent{\sc Proof of Claim.}}{\hfill $\boxtimes\hspace{2mm}$\end{trivlist}}

\renewcommand{\phi}{\varphi}
\usepackage{booktabs}
\usepackage{multirow}
\usepackage{makecell}
\usepackage{wrapfig}
\usepackage{graphicx}
\newcommand{\T}{{\sf T}}
\newcommand{\F}{{\sf F}}

\author{Junli Jiang
}
\address{Institute of Logic and Intelligence, Southwest University, Chongqing, China}
\ead{walk08@swu.edu.cn}
\author{Pavel Naumov}
\address{University of Southampton, Southampton, the United Kingdom}
\ead{pgn2@cornell.edu}

\begin{document}

\begin{frontmatter}
\title{Responsibility Gap and Diffusion in Sequential Decision-Making Mechanisms}

\begin{abstract}
Responsibility has long been a subject of study in law and philosophy. More recently, it became a focus of AI literature. 
The article investigates the computational complexity of two important properties of responsibility in collective decision-making: diffusion and gap. It shows that the sets of diffusion-free and gap-free decision-making mechanisms are $\Pi_2$-complete and $\Pi_3$-complete, respectively. At the same time, the intersection of these classes is $\Pi_2$-complete. 
\end{abstract}
\end{frontmatter}


\section{Introduction}

\subsection{Motivation}

From robotic surgical assistants and war robots to self-driving cars and automated stockbrokers, AI agents are increasingly contributing to significant decisions in everyday life. The outcomes of such decisions are often determined by the actions of multiple human and artificial agents.  To promote trustworthiness, reliability, and social acceptance of AI systems, it is important to establish individual accountability of all agents (human or otherwise) involved in collective decision-making\footnote{At the present time, the legal and moral responsibility for the actions of artificial agents should likely be passed on to their owners and manufacturers.}. To achieve this goal, we need to design collective decision-making mechanisms that promote individual accountability.

When designing such mechanisms, two distinct obstacles should be overcome: {\em responsibility gap} and {\em diffusion of responsibility}. The former refers to situations when an undesirable outcome happens and no single agent is responsible for it. The latter refers to similarly undesirable outcomes in which multiple agents are potentially responsible. Although diffusion of responsibility might appear not to be an obstacle but a blessing (the more agents to blame, the more accountability there is), it is not quite so. Diffusion of responsibility often leads to ``circle of blame'' and ``bystander effect''.

There are at least three different research directions associated with the decision-making mechanisms that promote individual accountability: existence, verification, and design. The first of them is to decide if there are mechanisms that completely avoid diffusion and/or gaps. The answer to this question is unquestionably yes: if a single agent (a ``dictator'') is solely making the decision, then this agent is the one and only one who is responsible for the outcome. Several existing papers studied, in various settings, whether non-dictatorial mechanisms can completely avoid a responsibility gap. Braham and van Hees~\cite{bh18ej} considered decision mechanisms in the ``discursive dilemma'' setting, where all agents simultaneously rank available options based on certain pre-determined criteria. They introduced a concept of ``fragmentation'' of responsibility, which is meant to capture the situations in which different parties are responsible for different aspects of the outcome. This notion is different from the diffusion of responsibility that we have discussed earlier. Braham and van Hees proved that, in the discursive dilemma setting, any gap-free mechanism that avoids fragmentation must be a dictatorship. Duijf and Van De Putte~\cite{dp22scw} gave a characterisation of gap-freeness in the discursive dilemma setting. 

Naumov and Tao~\cite{nt25ijcai} considered a more general class of concurrent decision-making mechanisms where the decision is made in multiple steps and all agents contribute concurrently (independently) at each step. They introduced a notion of an ``elected dictatorship'', which is such a concurrent decision-making mechanism that, along each decision path, there is a step at which one of the agents has a strategy to guarantee the positive and the negative outcome of the decision-making process. They proved that, in the perfect information case, gap-free mechanisms are exactly elected dictatorships. In the imperfect information setting, the class of gap-free mechanisms is strictly between what they call ``epistemic'' elected dictatorships and ``semi-epistemic'' elected dictatorships.

Although the above results may suggest that decision-making mechanisms can be gap-free only in special, degenerate cases, this conclusion holds only for the particular classes of mechanisms examined in those papers. As we will see in the current work, there are many non-trivial general case mechanisms that are gap-free and diffusion-free. The second research direction is to develop efficient procedures for verifying if a given mechanism admits a responsibility gap or diffusion of responsibility, or to prove that such procedures do not exist. The current article contributes to this direction. It considers a class of what we call ``sequential'' decision-making mechanisms where all agents act just once and in a fixed order. We show that even in this very restricted case, there are no efficient procedures for verifying that a mechanism is gap-free or diffusion-free. Namely, we prove that the class of gap-free mechanisms is $\Pi_3$-complete and the class of diffusion-free mechanisms is $\Pi_2$-complete. Note that, in general, the intersection of $\Pi_3$- and $\Pi_2$-hard sets is $\Pi_3$-hard. In spite of this, in this article, we show that the class of mechanisms that are simultaneously gap-free and diffusion-free is only $\Pi_2$-complete. This unexpected observation shows that gap-freeness and diffusion-freeness are two connected concepts. 

The third direction is actually designing decision-making mechanisms that promote individual accountability by eliminating (or at least minimising) the responsibility gap and diffusion. We leave this line of research for the future.

\subsection{Counterfactual Responsibility}

When the second author of this article was still a child, the author's uncle bought his first car. Soon after that, the uncle was approaching a road intersection near his home. The light turned yellow. Trying to be a law-abiding citizen, the uncle pushed the brakes. As his car started to slow down, he suddenly felt a force that 
smashed the rear part of the car, 
broke his neck, and
pushed his vehicle through the intersection. That force was\dots\ a lorry driving behind him. 

\begin{table}
\centering
    \begin{tabular}{c|cc}
        \toprule
\multirow{2}{*}{\makecell{\vspace{-2.5mm}\\Uncle}}    & \multicolumn{2}{c}{Lorry Driver}\\
        \cmidrule{2-3} 
    & brake & continue\\
        \midrule

         brake& collision & collision \\
        continue  \\
        \bottomrule
    \end{tabular}
\caption{Yellow Light mechanism}
\label{table:uncle}
\end{table}
Fortunately, the uncle recovered. But looking back, it is hard not to think that, as a nuclear physicist, he must have known that the lorry behind him had far greater inertia than his car—and that, whether the lorry's driver decided to brake at the yellow light or not, the driver could not have prevented the collision. In this article, we model this situation using a {\em decision mechanism} shown in Table~\ref{table:uncle}. This mechanism describes the outcome of the interaction of two {\em agents}: the uncle and the lorry driver. Both of them have two possible {\em actions}: to ``continue'' driving on yellow light or to ``brake''. By an {\em action profile} we mean one of four possible combinations of their actions. In the table, different cells correspond to different action profiles. As the table shows, the vehicles collide under action profiles where the uncle chooses to ``brake'', no matter what action is taken by the lorry driver.

In this article, we discuss the responsibility of agents. Responsibility is a broad term that has multiple meanings. It has been extensively studied by philosophers and legal scholars. 
The most commonly referred definition of responsibility in philosophy was suggested by  Frankfurt:
\begin{quote}
{\em \dots a person is morally responsible for what he has done only if he could have done otherwise.}\hfill--Frankfurt~\cite{f69tjop}    
\end{quote}
Frankfurt's definition is also known as the principle of ``alternative possibilities''~\cite{w17} or ``counterfactual'' possibilities~\cite{c15cop}. The same counterfactual approach could be used to define causality~\cite{lewis13,h16,bs18aaai}.  In this work, we will only consider Frankfurt's definition of responsibility. Following recent works in AI~\cite{ydjal19aamas,nt19aaai,nt20aaai,bfm21ijcai,s24aaai}, we interpret ``could have done otherwise'' as the person having a {\em strategy} (possible action) which {\em guarantees} the prevention of the outcome no matter how the other agents act. We refer to the responsibility defined this way as ``counterfactually responsibility'' or just ``responsibility''.

In our example, the uncle is counterfactually responsible for the collision because he would have guaranteed that the collision was avoided if he had chosen to ``continue'' driving on the yellow light, see Table~\ref{table:uncle}. At the same time, the lorry driver is not counterfactually responsible for the collision because there is no action that he could have taken that would have guaranteed that the collision would have been avoided.

\begin{table}
\centering
    \begin{tabular}{c|cc}
        \toprule
\multirow{2}{*}{\makecell{\vspace{-2.5mm}\\Factory A}}    & \multicolumn{2}{c}{Factory B}\\
        \cmidrule{2-3} 
   & pollute & don't pollute\\
        \midrule
 
        pollute & dead \\
        don't pollute &   &  \\
        \bottomrule
    \end{tabular}
\caption{Pollution mechanism}
\label{table:halpern}
\end{table}

\subsection{Diffusion of Responsibility}

In the discussed above Yellow Light mechanism, there is a {\em single responsible person} each time a collision happens. In many other cases, the situation is more complicated. Indeed, let us consider another example, originally proposed by Halpern~\cite{h16}. In this example, two factories, A and B, might decide to pollute a river. The fish in the river dies only if both factories choose to pollute, see Table~\ref{table:halpern}. Consider the action profile under which both factories simultaneously decide to pollute the river. As a result, the fish dies. In such a situation, {\em both} factories are counterfactually responsible for the death of the fish because each of them had a strategy (``don't pollute'') that would guarantee the survival of the fish no matter what the other company does, see Table~\ref{table:halpern}. When two different agents are responsible at the same time, we say that the decision mechanism admits {\em diffusion of responsibility}.  Diffusion of responsibility is a widely studied effect in social sciences~\cite{ms75pspb,fzg02pspb,llw22ichssr}, law~\cite{i20ijcr,rkfc22jrcd}, ethics~\cite{bb22aie}, and neuroscience~\cite{fdlglk16hbm}. It is an underlying cause of the ``bystander effect'' when multiple people observe something bad happening and leave it to others to intervene. In our example, each of the factories might decide to pollute and blame the other for the death of the fish, creating a ``circle of blame''. In general, the admissibility of diffusion is an undesirable property of a decision mechanism. If the mechanism does not allow diffusion of responsibility, no matter what actions are taken by the agents, then we say that the mechanism is {\em diffusion-free}. An example of a diffusion-free mechanism is the Yellow Light mechanism from Table~\ref{table:uncle}.

\subsection{Responsibility Gap}

In the two examples that we have considered so far, each time when something goes wrong, there is at least one responsible person. This is not true for our next example, which is a heavily modified version of the one suggested by Zultan, Gerstenberg, and Lagnado~\cite{zgl12cognition}. Imagine that a team consisting of Ann, Bob, and Charles
arrives at an international salsa competition. By the rules of the competition, at the opening ceremony, each team must be represented by a female and a male dancer wearing costumes of the same colour. Each of the team members has red, white, and blue dance costumes. Unfortunately, Ann (female), Bob (male), and Charles (male) forget to coordinate and appear at the opening in red, white, and blue costumes, respectively. As a result, they are disqualified from the competition. Note that, assuming that they pick dresses simultaneously (or at least independently), none of the team members has a strategy (choice of the dress colour) that would guarantee that the team is not disqualified, no matter what dresses are chosen by the other members of the team. Thus, none of the team members is counterfactually responsible for the disqualification. 
When none of the agents is (counterfactually) responsible, the decision mechanism has a {\em responsibility gap}.

The responsibility gap, also known as the {\em responsibility void}, has been discussed in the literature in many different contexts, but with essentially the same meaning. It is often used when discussing the ``moral agency'' of AI systems. If an AI system does not have a moral agency, then it cannot bear responsibility for its decision, which often creates a responsibility gap~\cite{m04eit,ct15pt,bhlmmp20ai,c20see,g20eit,sm21pt,t21pt,k22eit,o23pt,hv23synthese}. This term is also used in discussions that focus on group responsibility when it is hard to ascribe responsibility to an individual group member~\cite{l21pt,ygcsnj21ieee,ygsdjn21aamas,dy23as,ygsdjnr23ais}. Finally, similarly to the current article, this term is used in the context of collective decision-making mechanism design~\cite{bh11pq,d18pss,bh18ej,dp22scw,d22}. Shi~\cite{s24aaai} introduced the notion of a ``higher-order'' responsibility gap in a tree-like extensive game setting. 

In most situations, it is desirable to employ decision mechanisms that do not have responsibility gaps. We refer to such mechanisms as {\em gap-free}. The discussed earlier Yellow Light mechanism and the Pollution mechanism are gap-free.

\subsection{Decision Order}

A common way to reduce or totally eliminate the responsibility gap is to request that, instead of taking actions simultaneously (independently), the agents take them in a certain order. In such a setting, we assume that, when choosing an action, each agent knows all prior actions.

Suppose that, in our international salsa competition example, the choices are made in the order A, B, C. That is, Ann picks her costume first, then Bob, then Charles. Because Bob knows Ann's choice, he has a strategy to prevent disqualification by choosing a costume of the same colour as Ann's. Thus, under this order, Bob is counterfactually responsible for disqualification under each action profile that disqualifies their team. Hence, under the order A, B, C, the mechanism is gap-free. At the same time, Charles can also pick a costume of the same colour as Ann. Thus, he is also counterfactually responsible each time the team is disqualified. Therefore, the order A, B, C eliminates the responsibility gap, but it introduces the diffusion of responsibility. One can easily see that the order B, A, C also eliminates the responsibility gap and introduces the diffusion of responsibility (this time, between Ann and Charles).

Let us now consider the order B, C, A. Note that because the rules of the competition require a female and a male dancer dressed in costumes of the same colour, neither Bob nor Charles has a strategy to prevent disqualification. At the same time, Ann has such a strategy -- she needs to choose a costume that matches either Bob's or Charles' costume, which she can always do because she makes her choice last. Thus, under the order B, C, A, if the team is disqualified, then Ann is the only person counterfactually responsible. In other words, the order B, C, A eliminates the responsibility gap without introducing the diffusion of responsibility.

\subsection{Order is not a Panacea}

The introduction of an order in which the decisions are made does not always eliminate the responsibility gap. To observe this, let us consider the Clemency mechanism discussed in~\cite{s24aaai}. In the US, if all other legal means are exhausted, a prisoner on death row can petition for clemency. In most states, granting clemency is a joint decision of the state clemency board and the governor of the state. In many of these states, the board and the governor must both support clemency in order for the prisoner not to be executed, see Table~\ref{table:shi}.

\begin{table}
\centering
    \begin{tabular}{c|cc}
        \toprule
\multirow{2}{*}{\makecell{\vspace{-2.5mm}\\Board}}    & \multicolumn{2}{c}{Governor}\\
        \cmidrule{2-3} 
   & don't grant & grant\\
        \midrule
 
        don't support & executed & executed \\
        support & executed  &  \\
        \bottomrule
    \end{tabular}
\caption{Clemency mechanism}
\label{table:shi}
\end{table}
Imagine a situation when the board and the governor simultaneously decide not to support the clemency petition, and the prisoner is executed. Suppose that later a DNA analysis shows that the prisoner is innocent. Which of the two parties, the board or the governor, should be morally responsible for the death of an innocent person? 
One can argue that each of the two parties took the action that guaranteed the death of the prisoner and, thus, should be held responsible for the execution. This constitutes so-called ``seeing-to-it'' form of responsibility~\cite{b11jal,b11jpl,nt23apal} studied in STIT logic~\cite{bp90krdr,h01,h95jpl,hp17rsl,ow16sl}.
However, according to Frankfurt's principle of alternative possibilities, neither of them is counterfactually responsible because neither of the parties had a strategy that guarantees that the prisoner is not executed no matter what the other party does. Thus, this decision mechanism has a {\em responsibility gap}. In practice, the board makes its decision first; if the board supports the clemency petition, the case goes to the governor for the final decision. If the board makes the decision first, it has no strategy to prevent the execution of an innocent person. Thus, it is not counterfactually responsible for the execution. If the board has decided not to support the petition, then the governor has no power to overrule the board's decision. Hence, the governor is also not counterfactually responsible for the execution. As a result, each time the board does not support the petition, the prisoner is executed, and no party is responsible for this. Therefore, the introduction of the order ``the board, the governor'' does not eliminate the responsibility gap. The order ``the governor, the board'' does not eliminate the responsibility gap either due to the symmetry of the setting, see Table~\ref{table:shi}.   

The above example shows that the responsibility gap can not always be eliminated by introducing an order in which the decisions are taken. Note also that the diffusion of responsibility is not always {\em intoduced} by the order. As our Pollution example in Table~\ref{table:halpern} shows, sometimes the diffusion might already exist in the setting when all decisions are made simultaneously. If the diffusion exists in the simultaneous setting, it cannot be eliminated by an order. Indeed, if two distinct agents had strategies to prevent the undesired outcome in the simultaneous setting, then they still can use these strategies to achieve the same goal in the ``ordered'' setting.

\subsection{Contribution and Outline}\label{contribution section}

A crucial factor in the analysis of the computational complexity of mechanism design is the chosen way to represent the mechanism. In the case when the mechanism has $n$ agents, an {\em explicit} representation of the mechanism by a table like Table~\ref{table:uncle}--\ref{table:shi} can be exponential in terms of $n$. Shi studied tree-like explicit representations and observed that the presence of the gap can be verified in polynomial time~\cite{s24aaai}. 
Alternatively, a mechanism can be represented {\em symbolically} by a Boolean-valued formula with $n$ arguments. In many cases, the symbolic representation might be significantly more compact. When the results from~\cite{s24aaai} are applied to the symbolic representation of the mechanism, they give an exponential upper limit on the complexity of deciding if a mechanism is gap-free. 
We are not aware of any discussion in the literature on the complexity of verifying that a mechanism is diffusion-free.

In this article, we study the mechanism in {\em symbolic} representation. In such a setting, we investigate the complexity of deciding if a given decision mechanism, under a given order of agents, is diffusion-free and/or gap-free. We show that ``diffusion-free'' is $\Pi_2$-complete problem. At the same time, ``gap-free'' is a $\Pi_3$-complete problem. Recall that both diffusion-free and gap-free are desirable properties of the decision-making mechanisms. Thus, ideally, we would want a mechanism to be ``diffusion-and-gap-free''. As an intersection of a $\Pi_2$ and $\Pi_3$ sets, the ``diffusion-and-gap-free'' problem must belong to the class $\Pi_3$. However, as it turns out, it is not $\Pi_3$-complete. In fact, in this work, we show that ``diffusion-and-gap-free'' is a $\Pi_2$-complete problem. 

The rest of the article is organised as follows. First, we introduce terminology related to Boolean formulae. Then, we define the formal setting in which we will study the complexity of responsibility. In Section~\ref{Diffusion-Free Mechanisms section} and Section~\ref{Gap Problem section}, we study the complexity of the diffusion-free and gap-free classes, respectively. In  Section~\ref{Responsibility-Free Mechanisms section}, we introduce the auxiliary notion of ``responsibility-free'' mechanisms. Finally, in Section~\ref{Gap and Diffusion Problem section}, we prove that the ``diffusion-and-gap-free'' problem is $\Pi_2$-complete. The last section concludes.

\section{Technical Results}

\subsection{Boolean Formulae}

In this article, we denote Boolean {\em values} by 0 and 1. In the syntax, they are represented by constants $\F$ and $\T$. By a {\em Boolean formula} we mean an expression built out of propositional variables and constants $\F$ and $\T$ using propositional connectives: conjunction $\wedge$, disjunction $\vee$, negation $\neg$, and implication $\to$. A {\em quantified Boolean formula} can additionally use universal $\forall$ and existential $\exists$ Boolean quantifiers.

By a {\em valuation} we mean an arbitrary mapping of propositional variables into Boolean values $0$ and $1$. The Boolean value of a quantified Boolean formula under a valuation is defined recursively in the standard way. We say that quantified Boolean formulae $Q_1$ and $Q_2$ are {\em semantically equivalent} if they have the same value under each valuation and denote this by $Q_1\equiv Q_2$. For instance, $\forall p\neg (p\wedge q)\equiv\neg q$.

\subsection{Formal Setting}

In this subsection, we describe the formal setting in which we study diffusion-free and gap-free properties of the decision mechanisms. As mentioned in Section~\ref{contribution section}, an {\em explicit} representation of a mechanism by a table is exponentially large in terms of the number of agents. In this article, we use a {\em symbolic} representation. For example, for the Pollution mechanisms in Table~\ref{table:halpern}, we can represent the actions of Factory A and Factory B by Boolean values: $1$ (pollute) and $0$ (do not pollute). We use variable $v_a$ to denote the action taken by Factory A and variable $v_b$ to denote the action taken by Factory B. Note that Table~\ref{table:halpern} specifies a {\em deontic} (related to obligations) constraint on action profiles. In other words, it says which action profiles are not permissible because they kill the fish. In Table~\ref{table:uncle}, not permissible action profiles are those that result in a collision. In Table~\ref{table:shi}, not permissible\footnote{In this case, the constraint captures a moral obligation that exists under the assumption that the prisoner is innocent or under the assumption that the death penalty itself is immoral.} profiles lead to the execution of the prisoner.
In this article, we study the agent's responsibility for violating the deontic constraint. We express deontic constraint as a Boolean formula (without quantifiers). In the case of the Pollution mechanism, the deontic constraint can be captured by the Boolean formula $\gamma(v_a,v_b)=\neg (v_a\wedge v_b)$.

\begin{definition}\label{mechanism definition}
A (sequential decision-making) mechanism is a tuple $(n,\mathbf{v},\gamma)$, where 
\begin{enumerate}
    \item integer $n\ge 0$ is the number of ``agents'',
    \item $\mathbf{v}=\{\mathbf{v}_i\}_{1\le i\le n}$ is a family of disjoint ordered sets of Boolean variables,
    \item $\gamma$ is a ``deontic constraint'' Boolean formula without quantifiers whose variables belong to the set $\cup \mathbf{v}$.
\end{enumerate}
\end{definition}
Because our goal is to study diffusion responsibility and the responsibility gap for a fixed order, in the above definition, we essentially stipulate the order as a part of the formal definition of the mechanism. We represent agents by numbers 1 through $n$ and assume that they act in ascending order. In our Pollution mechanism example, each taken action was represented by a single Boolean variable ($v_a$ or $v_b$). To accommodate the mechanisms with agents that might have more than two actions, in the above definition, we encode each action by a set of Boolean variables. For example, in the Salsa Competition mechanism, the three possible actions (``red'', ``white'', ``blue'') of the first agent (Ann) can be represented by an ordered set $\mathbf{v}_1$ consisting of two Boolean variables: $\mathbf{v}_1=\{v^1_1,v^2_1\}$. These two Boolean variables can have $2\times 2=4$ possible combinations of values. To represent the three possible actions, two of these values could be assumed to represent the same action.

\begin{definition}
For any given mechanism $(n,\mathbf{v},\gamma)$, an {\bf\em action profile} is sequence $\mathbf{s}_1,\dots,\mathbf{s}_n$ of tuples of Boolean values  such that $\mathbf{s}_i\in\{0,1\}^{|\mathbf{v}_i|}$.   
\end{definition}

Intuitively, an action profile assigns an action to each agent. The tuple $\mathbf{s}_i$ in the above definition assigns a Boolean value to each Boolean variable in the ordered set $\mathbf{v}_i$. In this article, we use substitutions of the form $\phi[\mathbf{s}_i/\mathbf{v}_i]$, where $\phi$ is a Boolean formula with quantifiers. 
Formula $\phi[\mathbf{s}_i/\mathbf{v}_i]$ is a result of the replacement of each {\em free} occurrence of a variable from ordered set $\mathbf{v}_i$ in formula $\phi$ with the Boolean constant $\F$ if tuple $\mathbf{s}_i$ assigns value $0$ to this variable and constant $\T$ if it assigns value 1.  For example, $\forall p(p\to q)[\{0,1\}/\{p,q\}]$ is the formula $\forall p(p\to \T)$.
If $i_1,\dots,i_k$ are distinct agents, then by  
$\phi[\mathbf{s}_{i_1},\dots,\mathbf{s}_{i_k}/\mathbf{v}_{i_1},\dots,\mathbf{v}_{i_k}]$ we denote the Boolean formula 
$\phi[\mathbf{s}_{i_1}/\mathbf{v}_{i_1}]\dots[\mathbf{s}_{i_k}/\mathbf{v}_{i_k}]$.

\begin{definition}\label{responsibility}
For any given mechanism $(n,\mathbf{v},\gamma)$, an agent $i$ is {\bf\em responsible} under an action profile $\mathbf{s}_1,\dots,\mathbf{s}_n$ if 
\begin{enumerate}
    \item $\gamma[\mathbf{s}_1,\dots,\mathbf{s}_n/\mathbf{v}_1,\dots,\mathbf{v}_n]\equiv\F$,
    \item there is a tuple $\mathbf{t}\in\{1,0\}^{|\mathbf{v}_i|}$ such that $\gamma[\mathbf{s}'_1,\dots,\mathbf{s}'_n/\mathbf{v}_1,\dots,\mathbf{v}_n]\equiv\T$ for each action profile $\mathbf{s}'_1,\dots,\mathbf{s}'_n$ where
    \begin{enumerate}
        \item $\mathbf{s}'_j=\mathbf{s}_j$ for each agent $j<i$,
        \item $\mathbf{s}'_i=\mathbf{t}$.
    \end{enumerate}
\end{enumerate}
\end{definition}
Intuitively, item~1 of the above definition states that the deontic constraint is not satisfied. Item~2 states that agent $i$, acting after agents 1 through $i-1$, had a strategy to satisfy the constraint no matter what agents $i+1$ through $n$ would have decided to do after agent $i$ acts. 

It will be convenient to use notations $\exists \mathbf{v}_i\gamma$ and $\forall \mathbf{v}_i\gamma$ that combine multiple quantifiers into one. For example, if $\mathbf{v}_i=\{v^1_i,\dots,v^{|\mathbf{v}_i|}_i\}$, then $\exists \mathbf{v}_i\gamma$ is a short notation for $\exists v^1_i\dots\exists v^{|\mathbf{v}_i|}_i\gamma$. In the next two sections, we will be referring to formula ${\sf cf_i}$ defined as follows for each $i\le n$:
\begin{equation}\label{cf}
{\sf cf}_i=
\exists \mathbf{v}_i\forall \mathbf{v}_{i+1}\dots\forall \mathbf{v}_{n}\gamma.     
\end{equation}
Note that this formula contains bounded quantifiers over variables in ordered sets $\mathbf{v}_i$, $\mathbf{v}_{i+1}$, \dots, $\mathbf{v}_{n}$, leaving variables in ordered sets 
$\mathbf{v}_1$, \dots, $\mathbf{v}_{i-1}$ free. As a reminder of this, we will often write ${\sf cf}_i(\mathbf{v}_1,\dots, \mathbf{v}_{i-1})$ instead of~${\sf cf}_i$. Formula ${\sf cf}_i(\mathbf{v}_1,\dots, \mathbf{v}_{i-1})$ expresses the {\em counterfactual} statement: ``for the given values of $\mathbf{v}_1,\dots, \mathbf{v}_{i-1}$, agent $i$ has a strategy to guarantee that $\gamma$ has value $\T$ no matter what are the actions of the agents $i+1$, \dots, $n$''.

The next lemma follows from Definition~\ref{responsibility}. Note that the following formula: $(\neg \gamma \wedge {\sf cf}_i)[\mathbf{s}_1,\dots,\mathbf{s}_n/\mathbf{v}_1,\dots,\mathbf{v}_n]$ has no free propositional variables. Thus, it is semantically equivalent to either $\F$ or to $\T$.

\begin{lemma}\label{cf lemma}
$(\neg \gamma \wedge {\sf cf}_i)[\mathbf{s}_1,\dots,\mathbf{s}_n/\mathbf{v}_1,\dots,\mathbf{v}_n]\equiv\T$ 
if and only if
agent $i$ is responsible under an action profile $\mathbf{s}_1,\dots,\mathbf{s}_n$.    
\end{lemma}

\subsection{Diffusion-Free Mechanisms}\label{Diffusion-Free Mechanisms section}

In this section, we define the class of diffusion-free mechanisms and show that it is $\Pi_2$-complete.

\begin{definition}\label{DF definition}
Set {\sf DF} (``diffusion-free'') contains all mechanisms $(n,\mathbf{v},\gamma)$  for which under each profile  $\mathbf{s}_1,\dots,\mathbf{s}_n$ such that $\gamma[\mathbf{s}_1,\dots,\mathbf{s}_n/\mathbf{v}_1,\dots,\mathbf{v}_n]\equiv\F$
there is at most one responsible agent.
\end{definition}

\begin{lemma}
Set {\sf DF} belongs to class $\Pi_2$.    
\end{lemma}
\begin{proof}
For any mechanism $(n,\mathbf{v},\gamma)$,
consider the following closed (without free variables) formula
\begin{equation}\label{df definition}
{\sf df}(n,\mathbf{v},\gamma)=\forall\mathbf{v}_1...
\forall\mathbf{v}_n\!\left(\bigvee_{i\neq j}
({\sf cf}_i(\mathbf{v}_1,..., \mathbf{v}_{i-1})\wedge
{\sf cf}_j(\mathbf{v}_1,..., \mathbf{v}_{j-1}))\to\gamma
\right).
\end{equation}

\begin{claim}
${\sf df}(n,\mathbf{v},\gamma)$ is a $\Pi_2$-formula.    
\end{claim}
\begin{proof-of-claim}
Observe from equation~\eqref{cf} that formula ${\sf cf}_i$ is a $\Sigma_2$-formula for each agent $i\le n$. Thus, the formula ${\sf cf}_i\wedge {\sf cf}_j$ is also a $\Sigma_2$-formula for each agent $i,j\le n$. Then, the formula $\neg({\sf cf}_i\wedge {\sf cf}_j)$ is a $\Pi_2$-formula for each agent $i,j\le n$. 
Hence, the formula $(\bigwedge_{i\neq j}\neg
({\sf cf}_i\wedge
{\sf cf}_j))\vee\gamma$
is also a $\Pi_2$-formula.
Thus, the formula $\bigvee_{i\neq j}
({\sf cf}_i\wedge
{\sf cf}_j)\to\gamma$ is a $\Pi_2$-formula
because it is semantically equivalent to the previous formula.
Therefore, the formula ${\sf df}(n,\mathbf{v},\gamma)$ is also a $\Pi_2$-formula. 
\end{proof-of-claim}

Recall that ${\sf df}(n,\mathbf{v},\gamma)$ is a closed formula. Thus, it is semantically equivalent to either $\F$ or $\T$.
\begin{claim}
$(n,\mathbf{v},\gamma)\in {\sf DF}$ iff ${\sf df}(n,\mathbf{v},\gamma)\equiv\T$.
\end{claim}

\begin{proof-of-claim}
$(\Rightarrow):$ Suppose  ${\sf df}(n,\mathbf{v},\gamma)\equiv\F$. Thus, by equation~\eqref{df definition}, there is an action profile $\mathbf{s}_1,\dots,\mathbf{s}_n$ such that 
$$
\left(\bigvee_{i\neq j}
({\sf cf}_i(\mathbf{v}_1,\dots, \mathbf{v}_{i-1})\wedge
{\sf cf}_j(\mathbf{v}_1,\dots, \mathbf{v}_{j-1}))\to\gamma
\right)[\mathbf{s}_1,\dots,\mathbf{s}_n/\mathbf{v}_1,\dots,\mathbf{v}_n]\equiv\F.
$$
Hence, there must exist $i,j\le n$ such that $i\neq j$ and:
\begin{align*}
{\sf cf}_i[\mathbf{s}_1,\dots,\mathbf{s}_n/\mathbf{v}_1,\dots,\mathbf{v}_n]\equiv\T,\\
{\sf cf}_j[\mathbf{s}_1,\dots,\mathbf{s}_n/\mathbf{v}_1,\dots,\mathbf{v}_n]\equiv\T,\\
\gamma [\mathbf{s}_1,\dots,\mathbf{s}_n/\mathbf{v}_1,\dots,\mathbf{v}_n]\equiv\F.
\end{align*}
Thus, by Lemma~\ref{cf lemma}, agents $i$ and $j$ are both responsible under the action profile $\mathbf{s}_1,\dots,\mathbf{s}_n$ of the mecahism $(n,\mathbf{v},\gamma)$. Therefore, $(n,\mathbf{v},\gamma)\notin {\sf DF}$ by Definition~\ref{DF definition} and the assumption $i\neq j$.

\noindent
$(\Leftarrow):$ Suppose that $(n,\mathbf{v},\gamma)\notin {\sf DF}$. Thus, by Definition~\ref{DF definition}, there is an action profile  $\mathbf{s}_1,\dots,\mathbf{s}_n$ of the mechanism $(n,\mathbf{v},\gamma)$ and agents $i,j\le n$ such that $i\neq j$ and both agents are responsible under the profile $\mathbf{s}_1,\dots,\mathbf{s}_n$. Hence, by Lemma~\ref{cf lemma},
$$(\neg \gamma \wedge {\sf cf}_i)[\mathbf{s}_1,\dots,\mathbf{s}_n/\mathbf{v}_1,\dots,\mathbf{v}_n]\equiv\T
\text{ and }
(\neg \gamma \wedge {\sf cf}_j)[\mathbf{s}_1,\dots,\mathbf{s}_n/\mathbf{v}_1,\dots,\mathbf{v}_n]\equiv\T.
$$ Thus,
\begin{align*}
{\sf cf}_i[\mathbf{s}_1,\dots,\mathbf{s}_n/\mathbf{v}_1,\dots,\mathbf{v}_n]\equiv\T,\\
{\sf cf}_j[\mathbf{s}_1,\dots,\mathbf{s}_n/\mathbf{v}_1,\dots,\mathbf{v}_n]\equiv\T,\\
\gamma [\mathbf{s}_1,\dots,\mathbf{s}_n/\mathbf{v}_1,\dots,\mathbf{v}_n]\equiv\F.
\end{align*}
Then,
$$
\left(\bigvee_{i\neq j}
({\sf cf}_i(\mathbf{v}_1,\dots, \mathbf{v}_{i-1})\wedge
{\sf cf}_j(\mathbf{v}_1,\dots, \mathbf{v}_{j-1}))\to\gamma
\right)[\mathbf{s}_1,\dots,\mathbf{s}_n/\mathbf{v}_1,\dots,\mathbf{v}_n]\equiv\F.
$$
Therefore, ${\sf df}(n,\mathbf{v},\gamma)\equiv\F$ by equation~\eqref{df definition}. 
\end{proof-of-claim}
The statement of the lemma follows from the two claims above.
\end{proof}

\begin{lemma}
Set {\sf DF} is $\Pi_2$-hard.    
\end{lemma}
\begin{proof} 
Consider any Boolean formula $\phi$ without quantifiers whose variables are divided into two disjoint sets, $\mathbf{x}$ and $\mathbf{y}$. Define mechanism $(n,\mathbf{v},\gamma)$ as follows:
\begin{enumerate}
    \item $n=2$,
    \item $\mathbf{v}_1=\mathbf{x}\cup\{z_1\}$ and  $\mathbf{v}_2=\mathbf{y}\cup\{z_2\}$, where $z_1$ and $z_2$ are any two distinct Boolean variables such that $z_1,z_2\notin \mathbf{x}\cup\mathbf{y}$,
    \item $
\gamma=(\neg \phi\wedge z_1)\vee z_2
$.
\end{enumerate}
Note that $\forall \mathbf{x} \exists \mathbf{y} \phi$ is a closed formula. Thus, it is semantically equivalent to either $\F$ or $\T$.
To prove the statement of the lemma, it suffices to establish the following claim.
\begin{claim}
$\forall \mathbf{x} \exists \mathbf{y} \phi\equiv\T$ iff $(n,\mathbf{v},\gamma)\in {\sf DF}$.  \end{claim}

\begin{proof-of-claim}
$(\Rightarrow):$  Suppose that $\forall \mathbf{x} \exists \mathbf{y} \phi\equiv\T$.
Hence, 
$$\forall \mathbf{x}\forall z_1 \exists \mathbf{y} (\phi\vee \neg z_1)\equiv\T.$$
Thus,
$\forall \mathbf{x}\forall z_1 \exists \mathbf{y}\exists z_2 (( \phi\vee \neg z_1)\wedge \neg z_2)\equiv\T$.
Then, by De Morgan's laws, 
$\forall \mathbf{x}\forall z_1 \exists \mathbf{y}\exists z_2 \neg((\neg \phi\wedge z_1)\vee z_2)\equiv\T$. Hence,
$\forall \mathbf{x}\forall z_1 \exists \mathbf{y}\exists z_2 \neg\gamma\equiv\T$ by the definition of $\gamma$.
Recall that the first agent chooses the values of variables in the set $\mathbf{v}_1=\mathbf{x}\cup\{z_1\}$ and the second agent chooses the values of variables in the set $\mathbf{v}_2=\mathbf{y}\cup\{z_2\}$. 
Thus, the first agent has no action that guarantees that the deontic constraint $\gamma$ is satisfied. Hence, by item~2 of Definition~\ref{responsibility}, the first agent is not responsible under any action profile of the mechanism $(n,\mathbf{v},\gamma)$. Therefore, $(n,\mathbf{v},\gamma)\in {\sf DF}$ by Definition~\ref{DF definition} and because the mechanism has only two agents.

\noindent
$(\Leftarrow):$ Suppose that
$\forall \mathbf{x} \exists \mathbf{y} \phi\equiv\F$. Hence, 
$\neg\forall \mathbf{x} \exists \mathbf{y} \phi\equiv\T$. Thus, by De Morgan's laws, 
$\exists \mathbf{x} \forall \mathbf{y} \neg \phi\equiv\T$.
Then, 
$\exists \mathbf{x}\exists z_1 \forall \mathbf{y} (\neg \phi\wedge z_1)\equiv\T$. Hence, 
$$\exists \mathbf{x}\exists z_1 \forall \mathbf{y}\forall z_2 ((\neg \phi\wedge z_1)\vee z_2)\equiv\T.$$
Then, $\exists \mathbf{x}\exists z_1 \forall \mathbf{y}\forall z_2 \gamma\equiv\T$ by the definition of formula $\gamma$. Thus, the first agent has an action that guarantees that the deontic constraint $\gamma$ is satisfied no matter what the action of the second agent is. Note that the second agent can guarantee the constraint $
\gamma=(\neg \phi\wedge z_1)\vee z_2
$ by simply choosing $z_2=1$. Hence, both agents have actions that guarantee that the constraint $\gamma$ is satisfied. Consider any action profile under which $z_1=z_2=0$. Note that the formula $\gamma$ has value 0 under this profile. Yet, as we have seen above, both agents have a strategy to prevent this. Therefore, $(n,\mathbf{v},\gamma)\notin {\sf DF}$ by Definition~\ref{DF definition}.
\end{proof-of-claim}
This concludes the proof of the lemma.
\end{proof}

The next theorem follows from the two lemmas above.
\begin{theorem}
The set ${\sf DF}$ is $\Pi_2$-complete.    
\end{theorem}

\subsection{Gap Problem}\label{Gap Problem section}
In this subsection, we define the class of gap-free mechanisms and show that it is $\Pi_3$-complete.

\begin{definition}\label{GF definition}
Set {\sf GF} (``gap-free'') contains all (sequential decision-making) mechanisms $(n,\mathbf{v},\gamma)$  such that under each action profile $\mathbf{s}_1,\dots,\mathbf{s}_n$ of the mechanism, if $\gamma[\mathbf{s}_1,\dots,\mathbf{s}_n/\mathbf{v}_1,\dots,\mathbf{v}_n]\equiv\F$,
then there is at least one responsible agent under the profile $\mathbf{s}_1,\dots,\mathbf{s}_n$.
\end{definition}

\begin{lemma}
Set {\sf GF} belongs to class $\Pi_3$.    
\end{lemma}
\begin{proof}
For any mechanism $(n,\mathbf{v},\gamma)$,
consider the following formula
\begin{equation}\label{gf formula}
{\sf gf}(n,\mathbf{v},\gamma)=\forall\mathbf{v}_1\dots
\forall\mathbf{v}_n\left(\neg\gamma\to\bigvee_{i}
{\sf cf}_i(\mathbf{v}_1,\dots, \mathbf{v}_{i-1})
\right).
\end{equation}
\begin{claim}
${\sf gf}(n,\mathbf{v},\gamma)$ is a $\Pi_3$-formula for each mechanism $(n,\mathbf{v},\gamma)$.   
\end{claim}
\begin{proof-of-claim}
Formula ${\sf cf}_i$ belongs to class $\Sigma_2$ by equation~\eqref{cf}. Therefore,  formula ${\sf gf}(n,\mathbf{v},\gamma)$ belongs to class $\Pi_3$ by equation~\eqref{gf formula}.  
\end{proof-of-claim}

Note that the formula ${\sf gf}(n,\mathbf{v},\gamma)$ has no free variables. Thus, it is semantically equivalent to either $\F$ or $\T$.

\begin{claim}
$(n,\mathbf{v},\gamma)\in {\sf GF}$ iff ${\sf gf}(n,\mathbf{v},\gamma)\equiv\T$.
\end{claim}
\begin{proof-of-claim}
$(\Rightarrow):$ Suppose that the ${\sf gf}(n,\mathbf{v},\gamma)\equiv\F$. Thus, by equation~\eqref{gf formula}, there is an action profile 
$\mathbf{s}_1,\dots,\mathbf{s}_n$
such that
\begin{equation}\label{23-jul-a}
\gamma [\mathbf{s}_1,\dots,\mathbf{s}_n/\mathbf{v}_1,\dots,\mathbf{v}_n]\equiv\F
\end{equation}
and, for each $i\le n$,
\begin{equation*}
{\sf cf}_i[\mathbf{s}_1,\dots,\mathbf{s}_n/\mathbf{v}_1,\dots,\mathbf{v}_n]\equiv\F.
\end{equation*}
Thus, for each $i\le n$,
\begin{equation*}
(\neg \gamma \wedge {\sf cf}_i)[\mathbf{s}_1,\dots,\mathbf{s}_n/\mathbf{v}_1,\dots,\mathbf{v}_n]\equiv\F.
\end{equation*}
Hence, by Lemma~\ref{cf lemma}, none of the agents is responsible under the action profile 
$\mathbf{s}_1,\dots,\mathbf{s}_n$. Therefore, $(n,\mathbf{v},\gamma)\in {\sf GF}$ by Definition~\ref{GF definition} and equation~\eqref{23-jul-a}.

\noindent
$(\Leftarrow):$ Suppose that $(n,\mathbf{v},\gamma)\notin {\sf GF}$. Thus, by Definition~\ref{GF definition}, there is an action profile $\mathbf{s}_1,\dots,\mathbf{s}_n$ such that
\begin{equation}\label{23-jul-b}
\gamma [\mathbf{s}_1,\dots,\mathbf{s}_n/\mathbf{v}_1,\dots,\mathbf{v}_n]\equiv\F
\end{equation}
and no agent is responsible under the profile $\mathbf{s}_1,\dots,\mathbf{s}_n$. 
The latter, by Lemma~\ref{cf lemma}, implies that
$(\neg \gamma \wedge {\sf cf}_i)[\mathbf{s}_1,\dots,\mathbf{s}_n/\mathbf{v}_1,\dots,\mathbf{v}_n]\equiv\F$ for each agent $i\le n$. Thus, ${\sf cf}_i[\mathbf{s}_1,\dots,\mathbf{s}_n/\mathbf{v}_1,\dots,\mathbf{v}_n]\equiv\F$ for each agent $i\le n$ by equation~\eqref{23-jul-b}. Therefore, ${\sf gf}(n,\mathbf{v},\gamma)\equiv\F$ by equations~\eqref{gf formula} and \eqref{23-jul-b}.
\end{proof-of-claim}
This concludes the proof of the lemma.
\end{proof}

\begin{lemma}
Set {\sf GF} is $\Pi_3$-hard.  
\end{lemma}
\begin{proof}
Consider any Boolean formula $\phi$ without quantifiers whose variables are divided into three disjoint sets, $\mathbf{x}$, $\mathbf{y}$, and $\mathbf{z}$. Define mechanism $(n,\mathbf{v},\gamma)$ as follows:
\begin{enumerate}
    \item $n=3$,
    \item $\mathbf{v}_1=\mathbf{x}$, $\mathbf{v}_2=\mathbf{y}\cup\{t\}$, 
    $\mathbf{v}_3=\mathbf{z}$,
    where $t$ is any Boolean variable such that $t\notin \mathbf{x}\cup\mathbf{y}\cup\mathbf{z}$,
    \item $\gamma=\phi\wedge t$.
\end{enumerate}
Note that $\forall \mathbf{x} \exists \mathbf{y}\forall\mathbf{z} \phi$ is a closed formula. Thus, it is semantically equivalent to either $\F$ or $\T$.
To prove the statement of the lemma, it suffices to establish the following claim.
\begin{claim}
$\forall \mathbf{x} \exists \mathbf{y}\forall\mathbf{z}\, \phi\equiv\T$ iff $(n,\mathbf{v},\gamma)\in {\sf GF}$.  
\end{claim}
\begin{proof-of-claim}
$(\Rightarrow):$  Let $\forall \mathbf{x} \exists \mathbf{y}\forall\mathbf{z}\, \phi\equiv\T$. Thus, $\forall \mathbf{x} \exists \mathbf{y}\exists t \forall\mathbf{z}\, (\phi\wedge t)\equiv\T$. Recall that  $\mathbf{v}_1=\mathbf{x}$ and $\mathbf{v}_2=\mathbf{y}\cup\{t\}$. Thus, no matter what action the first agent takes, the second agent can choose the values of variables in set $\mathbf{v}_2$ in such a way that the deontic constraint $\gamma=\phi\wedge t$ is satisfied under the resulting action profile. Hence, by Definition~\ref{responsibility}, the second agent is responsible under each action profile $\mathbf{s}_1,\dots,\mathbf{s}_n$ such that $\gamma[\mathbf{s}_1,\dots,\mathbf{s}_n/\mathbf{v}_1,\dots,\mathbf{v}_n]\equiv\F$. Therefore, $(n,\mathbf{v},\gamma)\in {\sf GF}$ by Definition~\ref{GF definition}. 

\noindent
$(\Leftarrow):$  Suppose 
$\forall \mathbf{x} \exists \mathbf{y}\forall\mathbf{z}\, \phi\equiv\F$. Thus,
$\exists \mathbf{x} \forall \mathbf{y}\exists\mathbf{z}\, \neg \phi\equiv\T$. Then, there is a tuple $\mathbf{x}_0\in \{0,1\}^{|\mathbf{x}|}$ such that
\begin{equation}\label{24-jul-a}
(\forall \mathbf{y}\exists\mathbf{z}\neg \phi)[\mathbf{x}_0/\mathbf{x}]\equiv \T.    
\end{equation}
Let us consider {\em any} action profile $\mathbf{s}_1,\mathbf{s}_2,\mathbf{s}_3$ under which the first agent chooses action $\mathbf{x}_0$ and the second agent assigns value $0$ to variable $t$. Note that
$\gamma[\mathbf{s}_1,\mathbf{s}_2,\mathbf{s}_3/\mathbf{v}_1,\mathbf{v}_2,\mathbf{v}_3]\equiv\F$
because $t$ has value $0$. Thus, by Definition~\ref{GF definition}, to prove that $(n,\mathbf{v},\gamma)\notin {\sf GF}$, it suffices to show that neither of the agents is responsible under the action profile $\mathbf{s}_1,\mathbf{s}_2,\mathbf{s}_3$.

The first and the third agents are not responsible because neither of them has a strategy to guarantee that the deontic constraint $\gamma=\phi\wedge t$ is satisfied. This is because neither of these agents has control of variable $t$.

The second agent has no strategy to guarantee that the constraint $\gamma=\phi\wedge t$ is satisfied due to equation~\eqref{24-jul-a} and because the first agent has chosen action $\mathbf{x}_0$.
\end{proof-of-claim}
This concludes the proof of the lemma.
\end{proof}

The next theorem follows from the two lemmas above.
\begin{theorem}
The set ${\sf GF}$ is $\Pi_3$-complete.    
\end{theorem}

\subsection{Responsibility-Free Mechanisms}\label{Responsibility-Free Mechanisms section}

In the previous two sections, we have shown that sets ${\sf DF}$ and ${\sf GF}$ are $\Pi_2$-complete and $\Pi_3$-complete, respectively. This implies that the intersection of these two sets is a $\Pi_3$-set. In the rest of this article, we will observe that the set ${\sf GF}\cap {\sf DF}$ is, perhaps surprisingly, a $\Pi_2$-complete set. To prove that this set is a $\Pi_2$-set, we need to express ``gap-and-diffusion-free'' as a $\Pi_2$-formula. To do this, it is convenient first to write an auxiliary  ``responsibility-free'' formula that expresses a completely ``irresponsible'' mechanism design: no agent is responsible under any action profile. We do this in the current section.

\begin{definition}\label{RF}
Set {\sf RF} (``responsibility-free'') contains all mechanisms $(n,\mathbf{v},\gamma)$  for which there are no agents responsible under any action profile.
\end{definition}

When discussing responsibility-free mechanisms, it is convenient 
for any mechanism $M=(n,\mathbf{v},\gamma)$  
to consider a ``partial'' mechanism in which several agents have already made their choices. For example, if we consider the Salsa Competition mechanism with the order A, B, C, then a partial mechanism could represent a situation when Ann has already picked a dress and, thus, there are only two agents left in the mechanism. To be precise, there are {\em three} different partial mechanisms in which Ann picked red, white, and blue dresses. In general, by
$M_k(\mathbf{s}_1,\dots,\mathbf{s}_k)$
we denote the partial mechanism in which the first $k$ agents already made their choices and these choices are $\mathbf{s}_1,\dots,\mathbf{s}_k$.

\begin{definition}\label{partial mechanism}
For any action profile $\mathbf{s}_1,\dots,\mathbf{s}_k$ of a mechanism $M=(n,\mathbf{v},\gamma)$ and any integer $k\le n$, let mechanism $M_k(\mathbf{s}_1,\dots,\mathbf{s}_k)$ be the tuple 
$$(n-k,\{\mathbf{v}_{k+i}\}_{1\le i\le n-k},\gamma[\mathbf{s}_1,\dots,\mathbf{s}_k/\mathbf{v}_1,\dots,\mathbf{v}_k]).$$
\end{definition}
If $k=0$, then $M_k$ is just the original mechanism $M$. At the same time, if $k=n$, then $M_k$ is the ``trivial mechanism'' that has no agents. Thus, this mechanism vacuously belongs to classes  {\sf DF} and {\sf RF}. In addition, the only action profile of this mechanism is the empty sequence and deontic constraint $\gamma[\mathbf{s}_1,\dots,\mathbf{s}_k/\mathbf{v}_1,\dots,\mathbf{v}_k]$ is a closed formula which is semantically equivalent either to $\T$ or to $\F$. In the former case, the deontic constraint is satisfied under the unique action profile of the mechanism. Thus, the mechanism belongs to class {\sf GF}. In the latter case, it does not.

Let us again suppose that $M$ is the  Salsa Competition mechanism with agents Ann, Bob, and Charles. Formally, they are referred to as the first, the second, and the third agents of this mechanism. The mechanism $M_1(\mathbf{s}_1)$, for any tuple $\mathbf{s}_1$, has only two agents, Bob and Charles. In this mechanism, Bob is referred to as the first agent. Similarly, Charles is the {\em third} agent in mechanism $M_0$, the {\em second} in mechanism $M_1$, and the {\em first} in mechanism $M_2$. This observation clarifies why the first and the second statements in the lemma below refer to agent $i$ and agent $i+1$, respectively. 

\begin{lemma}\label{AAA}
For any action profile $\mathbf{s}_{1},\dots,\mathbf{s}_{n}$ of mechanism $M=(n,\mathbf{v},\gamma)$, any $k\le n$ and any $i< n-k$, the following two statements are equivalent:
\begin{enumerate}
    \item agent $i$ is responsible under the profile 
$\mathbf{s}_{k+2},\dots,\mathbf{s}_{n}$ of the mechanism $M_{k+1}(\mathbf{s}_1,\dots,\mathbf{s}_k,\mathbf{s}_{k+1})$,
    \item agent $i+1$ is responsible under the profile 
$\mathbf{s}_{k+1},\mathbf{s}_{k+2},\dots,\mathbf{s}_{n}$ of the mechanism $M_{k}(\mathbf{s}_1,\dots,\mathbf{s}_k)$.
\end{enumerate}
\end{lemma}
\begin{proof}
The deontic constraint of the mechanism $M_{k+1}(\mathbf{s}_1,\dots,\mathbf{s}_{k},\mathbf{s}_{k+1})$ is the formula 
$$\gamma[\mathbf{s}_1,\dots,\mathbf{s}_{k},\mathbf{s}_{k+1}/\mathbf{v}_1,\dots,\mathbf{v}_{k},\mathbf{v}_{k+1}]$$
by Definition~\ref{partial mechanism}. In addition, by equation~\eqref{cf}, the statement ${\sf cf}_i$ for the mechanism 
$M_{k+1}(\mathbf{s}_1,\dots,\mathbf{s}_{k},\mathbf{s}_{k+1})$ is $$\exists\mathbf{v}_{(k+1)+i}\forall\mathbf{v}_{(k+1)+i+1}\dots\forall\mathbf{v}_n(\gamma[\mathbf{s}_1,\dots,\mathbf{s}_{k},\mathbf{s}_{k+1}/\mathbf{v}_1,\dots,\mathbf{v}_{k},\mathbf{v}_{k+1}]).$$
Then, by Lemma~\ref{cf lemma} and equation~\eqref{cf}, the first statement of the lemma is equivalent to
\begin{multline*}
(\neg \gamma[\mathbf{s}_1,\dots,\mathbf{s}_{k},\mathbf{s}_{k+1}/\mathbf{v}_1,\dots,\mathbf{v}_{k},\mathbf{v}_{k+1}]\wedge\\ \exists\mathbf{v}_{(k+1)+i}\forall\mathbf{v}_{(k+1)+i+1}\dots\forall\mathbf{v}_n(\gamma[\mathbf{s}_1,\dots,\mathbf{s}_{k},\mathbf{s}_{k+1}/\mathbf{v}_1,\dots,\mathbf{v}_{k},\mathbf{v}_{k+1}]))\\\
[\mathbf{s}_{k+2},\dots,\mathbf{s}_{n}/\mathbf{v}_{k+2},\dots,\mathbf{v}_{n}]\equiv\T.
\end{multline*}
Note that sets $\mathbf{v}_1,\dots,\mathbf{v}_n$ are disjoint by Definition~\ref{mechanism definition}. Thus, the above statement is equivalent to
\begin{multline}\label{10-aug-a}
(\neg \gamma[\mathbf{s}_1,\dots,\mathbf{s}_{k}/\mathbf{v}_1,\dots,\mathbf{v}_{k}]\wedge\\ \exists\mathbf{v}_{k+(i+1)}\forall\mathbf{v}_{k+(i+2)}\dots\forall\mathbf{v}_n(\gamma[\mathbf{s}_1,\dots,\mathbf{s}_{k}/\mathbf{v}_1,\dots,\mathbf{v}_{k}]))\\\
[\mathbf{s}_{k+1},\mathbf{s}_{k+2},\dots,\mathbf{s}_{n}/\mathbf{v}_{k+1},\mathbf{v}_{k+2},\dots,\mathbf{v}_{n}]\equiv\T.
\end{multline}
Observe by Definition~\ref{partial mechanism} that $\gamma[\mathbf{s}_1,\dots,\mathbf{s}_{k}/\mathbf{v}_1,\dots,\mathbf{v}_{k}]$ is the deontic constraint of the mechanism $M_{k}(\mathbf{s}_1,\dots,\mathbf{s}_k)$ and that the  formula
$$\exists\mathbf{v}_{k+(i+1)}\forall\mathbf{v}_{k+(i+2)}\dots\forall\mathbf{v}_n(\gamma[\mathbf{s}_1,\dots,\mathbf{s}_{k}/\mathbf{v}_1,\dots,\mathbf{v}_{k}])$$ is the statement ${\sf cf}_{i+1}$ for the mechanism $M_{k}(\mathbf{s}_1,\dots,\mathbf{s}_k)$ by equation~\eqref{cf}.
Hence, by Lemma~\ref{cf lemma}, statement~\eqref{10-aug-a} is equivalent to the second statement in the lemma.
\end{proof}

So far, we have been using the substitution notation $\phi[\mathbf{s}/\mathbf{v}]$ where $\mathbf{s}$ is a tuple of Boolean {\em values}. For example, if $\mathbf{v}=v_1,v_2$ and $\mathbf{s}=0,1$, then $$(v_1\to v_2)[\mathbf{s}/\mathbf{v}]=\F\to \T.$$ For the next definition, we need a {\em variable renaming} notation $\phi\langle\mathbf{v'}/\mathbf{v}\rangle$, where $\mathbf{v}$ and $\mathbf{v'}$ are two ordered sets of Boolean variables of the same size. The formula $\phi\langle\mathbf{v'}/\mathbf{v}\rangle$ is the result of renaming all variables from ordered set $\mathbf{v}$ into their counterpart in ordered set $\mathbf{v'}$. For instance, if $\mathbf{v}=v_1,v_2$ and $\mathbf{v}'=v_2,v_3$, then $(v_1\to v_2)\langle\mathbf{v'}/\mathbf{v}\rangle=v_2\to v_3$. If $\mathbf{v}_1,\dots,\mathbf{v}_k$ are disjoint ordered sets of Boolean variables, then by $\phi\langle\mathbf{v}'_1,\dots,\mathbf{v}'_k/\mathbf{v}_1,\dots,\mathbf{v}_k\rangle$ we mean the formula 
$
\phi\langle\mathbf{v}'_1/\mathbf{v}_1\rangle\dots\langle\mathbf{v}'_k/\mathbf{v'}_k\rangle
$.

We use the variable renaming notation to define the formula ${\sf rf}_k$ for each agent $k\le n$:
\begin{equation}\label{rf}
{\sf rf}_k=
\forall \mathbf{v}_{k+1}\dots \forall\mathbf{v}_{n}\forall\mathbf{v}'_{k+1}\dots\forall\mathbf{v}'_n
(\gamma\to \gamma\langle\mathbf{v}'_{k+1},\dots,\mathbf{v}'_{n}/\mathbf{v}_{k+1},\dots,\mathbf{v}_n\rangle). 
\end{equation}
Note that formula ${\sf rf}_k$ is not closed. It has free variables from the ordered sets $\mathbf{v}_1,\dots,\mathbf{v}_k$. 
Intuitively, the formula ${\sf rf}_k$ expresses the {\em uniformity condition}: if the values of the variables from the ordered sets $\mathbf{v}_{1},\dots,\mathbf{v}_{k}$ are fixed, then the value of formula $\gamma$ is the same no matter what are the values of variables from the ordered sets $\mathbf{v}_{k+1},\dots,\mathbf{v}_{n}$.
The next lemma explains the choice of the name for the formula ${\sf rf}_k$.

\begin{lemma}\label{RF lemma}
$M_k(\mathbf{s}_1,\dots,\mathbf{s}_k)\in  {\sf RF}$ iff ${\sf rf}_k[\mathbf{s}_1,\dots,\mathbf{s}_k/\mathbf{v}_1,\dots,\mathbf{v}_k]\equiv\T$.    
\end{lemma}
\begin{proof}
$(\Rightarrow):$
We prove this direction by backward induction on $k$. If $k=n$, then,
by equation~\eqref{rf}, formula ${\sf rf}_n$ is the tautology $\gamma\to\gamma$. Therefore, ${\sf rf}_n[\mathbf{s}_1,\dots,\mathbf{s}_n/\mathbf{v}_1,\dots,\mathbf{v}_n]\equiv\T$. 

Suppose that $k<n$. We consider the following three cases separately:

\vspace{1mm}\noindent
{\em Case 1}: 
$(\exists \mathbf{v}_{k+1}\forall \mathbf{v}_{k+2}\dots \mathbf{v}_n\gamma)[\mathbf{s}_1,\dots,\mathbf{s}_k/\mathbf{v}_1,\dots,\mathbf{v}_k]\equiv \T$. Thus,
$$
{\sf cf}_{k+1}[\mathbf{s}_1,\dots,\mathbf{s}_k/\mathbf{v}_1,\dots,\mathbf{v}_k]\equiv \T
$$
by equation~\eqref{cf}. Thus, by the assumption $M_k(\mathbf{s}_1,\dots,\mathbf{s}_k)\in  {\sf RF}$ of the lemma,  Definition~\ref{RF} and Lemma~\ref{cf lemma},
$$
\gamma[\mathbf{s}_1,\dots,\mathbf{s}_k/\mathbf{v}_1,\dots,\mathbf{v}_k][\mathbf{s}_{k+1},\dots,\mathbf{s}_{n}/\mathbf{v}_{k+1},\dots,\mathbf{v}_k]\equiv \T
$$
for each sequence of Boolean tuples $\mathbf{s}_{k+1},\dots,\mathbf{s}_{n}$. Note that sets $\mathbf{v}_1,\dots,\mathbf{v}_n$ are disjoint by Definition~\ref{mechanism definition}. Hence,
$
(\forall\mathbf{v}_{k+1} \dots \mathbf{v}_n\gamma)[\mathbf{s}_1,\dots,\mathbf{s}_k/\mathbf{v}_1,\dots,\mathbf{v}_k]\equiv \T
$.
Therefore, 
${\sf rf}_k[\mathbf{s}_1,\dots,\mathbf{s}_k/\mathbf{v}_1,\dots,\mathbf{v}_k]\equiv\T$
by equation~\eqref{rf}.

\vspace{1mm}\noindent
{\em Case 2}: 
$(\forall\mathbf{v}_{k+1}\forall \mathbf{v}_{k+2}\dots \mathbf{v}_n\neg\gamma)[\mathbf{s}_1,\dots,\mathbf{s}_k/\mathbf{v}_1,\dots,\mathbf{v}_k]\equiv \T$. Thus, 
$${\sf rf}_k[\mathbf{s}_1,\dots,\mathbf{s}_k/\mathbf{v}_1,\dots,\mathbf{v}_k]\equiv\T$$ by equation~\eqref{rf}.

\vspace{1mm}\noindent
{\em Case 3}: the following takes place:
\begin{align*}
&(\exists \mathbf{v}_{k+1}\forall \mathbf{v}_{k+2}\dots \mathbf{v}_n\gamma)[\mathbf{s}_1,\dots,\mathbf{s}_k/\mathbf{v}_1,\dots,\mathbf{v}_k]\equiv \F,\\
&(\forall\mathbf{v}_{k+1}\forall \mathbf{v}_{k+2}\dots \mathbf{v}_n\neg\gamma)[\mathbf{s}_1,\dots,\mathbf{s}_k/\mathbf{v}_1,\dots,\mathbf{v}_k]\equiv \F.
\end{align*}
Thus,
\begin{align*}
&(\forall \mathbf{v}_{k+1}\exists \mathbf{v}_{k+2}\dots \mathbf{v}_n\neg\gamma)[\mathbf{s}_1,\dots,\mathbf{s}_k/\mathbf{v}_1,\dots,\mathbf{v}_k]\equiv \T,\\
&(\exists\mathbf{v}_{k+1}\exists \mathbf{v}_{k+2}\dots \mathbf{v}_n\gamma)[\mathbf{s}_1,\dots,\mathbf{s}_k/\mathbf{v}_1,\dots,\mathbf{v}_k]\equiv \T.
\end{align*}
Hence, there exists a Boolean tuple $\mathbf{s}_{k+1}\in \{0,1\}^{|v_{k+1}|}$ such that
\begin{align*}
(\exists\mathbf{v}_{k+2}\dots \mathbf{v}_n\neg\gamma)[\mathbf{s}_1,\dots,\mathbf{s}_k,\mathbf{s}_{k+1}/\mathbf{v}_1,\dots,\mathbf{v}_k,\mathbf{v}_{k+1}]\equiv \T,\\
(\exists\mathbf{v}_{k+2}\dots \mathbf{v}_n\gamma)[\mathbf{s}_1,\dots,\mathbf{s}_k,\mathbf{s}_{k+1}/\mathbf{v}_1,\dots,\mathbf{v}_k,\mathbf{v}_{k+1}]\equiv \T.    
\end{align*}
Then,
$
{\sf rf}_{k+1}[\mathbf{s}_1,\dots,\mathbf{s}_k,\mathbf{s}_{k+1}/\mathbf{v}_1,\dots,\mathbf{v}_k,\mathbf{v}_{k+1}]\equiv \F
$
by equation~\eqref{rf}.
Thus, 
by the induction hypothesis,
$$M_{k+1}(\mathbf{s}_1,\dots,\mathbf{s}_k,\mathbf{s}_{k+1})\notin  {\sf RF}.$$
Therefore, 
$M_k(\mathbf{s}_1,\dots,\mathbf{s}_k)\notin  {\sf RF}$ by Lemma~\ref{AAA}.

\noindent
$(\Leftarrow):$  Suppose that 
${\sf rf}_k[\mathbf{s}_1,\dots,\mathbf{s}_k/\mathbf{v}_1,\dots,\mathbf{v}_k]\equiv\T$.
Thus, by equation~\eqref{rf} and Definition~\ref{partial mechanism}, deontic constraint  $\gamma[\mathbf{s}_1,\dots,\mathbf{s}_k/\mathbf{v}_1,\dots,\mathbf{v}_k]$ of mechanism $M_k(\mathbf{s}_1,\dots,\mathbf{s}_k)$ is either satisfied for each action profile or not satisfied for each action profile. Thus, by Definition~\ref{responsibility}, no agent can be responsible under any of the action profiles. Then, $M_k(\mathbf{s}_1,\dots,\mathbf{s}_k)\in  {\sf RF}$ by Definition~\ref{RF}.
\end{proof}

\subsection{Gap and Diffusion Problem}\label{Gap and Diffusion Problem section}

In this section, we prove that the intersection of sets {\sf GF} and {\sf DF} is $\Pi_2$-complete. We call this intersection {\sf GDF} (``gap-and-diffusion-free'').

\begin{definition}\label{GDF definition}
Set {\sf GDF} contains all such mechanisms $(n,\mathbf{v},\gamma)$  that for any action profile $\mathbf{s}_1,\dots,\mathbf{s}_n$ where $\gamma[\mathbf{s}_1,\dots,\mathbf{s}_n/\mathbf{v}_1,\dots,\mathbf{v}_n]\equiv\F$ there is exactly one agent resposponsible under profile $\mathbf{s}_1,\dots,\mathbf{s}_n$.
\end{definition}

\begin{lemma}\label{green layer lemma}
For any mechanism $M=(n,\mathbf{v},\gamma)$, any integer $k<n$, and any tuples $\mathbf{s}_1,\dots,\mathbf{s}_k$
such that
 $\exists \mathbf{v}_{k+1}\forall\mathbf{v}_{k+2}\dots \mathbf{v}_n (\gamma[\mathbf{s}_1,\dots,\mathbf{s}_k/\mathbf{v}_1,\dots,\mathbf{v}_k])\equiv\T$, the following two conditions are equivalent:
 \begin{enumerate}
    \item $M_k(\mathbf{s}_1,\dots,\mathbf{s}_k)\in  {\sf GDF}$,
     \item $M_{k+1}(\mathbf{s}_1,\dots,\mathbf{s}_{k},\mathbf{s}_{k+1})\in  {\sf RF}$ for each tuple $\mathbf{s}_{k+1}\in \{0,1\}^{\mathbf{v}_{k+1}}$.
 \end{enumerate}
\end{lemma}
\begin{proof}
\noindent$(1\Rightarrow 2):$
Suppose that $M_{k+1}(\mathbf{s}_1,\dots,\mathbf{s}_{k},\mathbf{s}_{k+1})\notin  {\sf RF}$ for some tuple $\mathbf{s}_{k+1}\in \{0,1\}^{\mathbf{v}_{k+1}}$. Thus, by Definition~\ref{RF}, there is an agent $i<n-k$ responsible under some profile   
$\mathbf{s}_{k+2},\dots,\mathbf{s}_{n}$ of the mechanism $M_{k+1}(\mathbf{s}_1,\dots,\mathbf{s}_{k},\mathbf{s}_{k+1})$.
Hence, by Lemma~\ref{AAA}, 
agent $i+1$ is responsible under the profile
$\mathbf{s}_{k+1},\mathbf{s}_{k+2},\dots,\mathbf{s}_{n}$ of the mechanism $M_k(\mathbf{s}_1,\dots,\mathbf{s}_k)$. Thus, by Definition~\ref{responsibility}, it means that
\begin{equation}\label{11-aug-a}
(\gamma[\mathbf{s}_1,\dots,\mathbf{s}_k/\mathbf{v}_1,\dots,\mathbf{v}_k])[\mathbf{s}_{k+1},\dots,\mathbf{s}_n/\mathbf{v}_{k+1},\dots,\mathbf{v}_n]\equiv\F    
\end{equation}
because the formula $\gamma[\mathbf{s}_1,\dots,\mathbf{s}_k/\mathbf{v}_1,\dots,\mathbf{v}_k]$ is the deontic constraint of mechanism $M_k(\mathbf{s}_1,\dots,\mathbf{s}_k)$.

Additionally, agent $1$ is also responsible under the profile
$\mathbf{s}_{k+1},\mathbf{s}_{k+2},\dots,\mathbf{s}_{n}$ of the mechanism $M_k(\mathbf{s}_1,\dots,\mathbf{s}_k)$ by 
the assumption  
$$\exists \mathbf{v}_{k+1}\forall\mathbf{v}_{k+2}\dots \mathbf{v}_n (\gamma[\mathbf{s}_1,\dots,\mathbf{s}_k/\mathbf{v}_1,\dots,\mathbf{v}_k])\equiv\T$$ of the lemma,
Definition~\ref{responsibility}, and statement~\eqref{11-aug-a}.
Hence, under the profile $\mathbf{s}_{k+1},\mathbf{s}_{k+2},\dots,\mathbf{s}_{n}$,
the mechanism $M_k(\mathbf{s}_1,\dots,\mathbf{s}_k)$ has at least two agents responsible. Therefore, $M_k(\mathbf{s}_1,\dots,\mathbf{s}_k)\notin  {\sf GDF}$ by Definition~\ref{GDF definition}.

\noindent$(2\Rightarrow 1):$
Suppose that 
\begin{equation}\label{11-aug-b}
M_k(\mathbf{s}_1,\dots,\mathbf{s}_k)\notin  {\sf GDF}.    
\end{equation}
Observe that the assumption 
$$\exists \mathbf{v}_{k+1}\forall\mathbf{v}_{k+2}\dots \mathbf{v}_n (\gamma[\mathbf{s}_1,\dots,\mathbf{s}_k/\mathbf{v}_1,\dots,\mathbf{v}_k])\equiv\T$$ 
of the lemma implies that agent 1 is responsible under each action profile $\mathbf{s}_{k+1},\dots,\mathbf{s}_n$ of the mechanism $M_k(\mathbf{s}_1,\dots,\mathbf{s}_k)$ such that 
$$
(\gamma[\mathbf{s}_1,\dots,\mathbf{s}_k/\mathbf{v}_1,\dots,\mathbf{v}_k])[\mathbf{s}_{k+1},\dots,\mathbf{s}_n/\mathbf{v}_1,\dots,\mathbf{v}_k]\equiv\F.
$$
Then, $M_k(\mathbf{s}_1,\dots,\mathbf{s}_k)\in  {\sf GF}$
by Definition~\ref{GF definition}. 
Thus,
$M_k(\mathbf{s}_1,\dots,\mathbf{s}_k)\notin  {\sf DF}$ by assumption~\eqref{11-aug-b}, Definition~\ref{GF definition}, Definition~\ref{DF definition}, and Definition~\ref{GDF definition}.
Hence, by Definition~\ref{DF definition}, at least two distinct agents are responsible under some profile $\mathbf{s}_{k+1},\dots,\mathbf{s}_n$ of the mechanism
$M_k(\mathbf{s}_1,\dots,\mathbf{s}_k)$. Thus, there is at least one agent $i\ge 2$ responsible under the profile $\mathbf{s}_{k+1},\dots,\mathbf{s}_n$ of the mechanism
$M_k(\mathbf{s}_1,\dots,\mathbf{s}_k)$.
Then, by Lemma~\ref{AAA}, agent $i-1$ is responsible under the profile $\mathbf{s}_{k+2},\dots,\mathbf{s}_n$ of the mechanism
$M_{k+1}(\mathbf{s}_1,\dots,\mathbf{s}_k,\mathbf{s}_{k+1})$.
Therefore, $M_{k+1}(\mathbf{s}_1,\dots,\mathbf{s}_k,\mathbf{s}_{k+1})\notin {\sf RF}$ by Definition~\ref{RF}.
\end{proof}

\begin{lemma}\label{no green layer lemma}
For any mechanism $M=(n,\mathbf{v},\gamma)$, any integer $k<n$, and any tuples $\mathbf{s}_1,\dots,\mathbf{s}_k$
such that
 $\exists \mathbf{v}_{k+1}\forall\mathbf{v}_{k+2}\dots \mathbf{v}_n (\gamma[\mathbf{s}_1,\dots,\mathbf{s}_k/\mathbf{v}_1,\dots,\mathbf{v}_k])\equiv\F$, the following two conditions are equivalent:
 \begin{enumerate}
    \item $M_k(\mathbf{s}_1,\dots,\mathbf{s}_k)\in  {\sf GDF}$,
     \item $M_{k+1}(\mathbf{s}_1,\dots,\mathbf{s}_{k},\mathbf{s}_{k+1})\in  {\sf GDF}$ for each tuple $\mathbf{s}_{k+1}\in \{0,1\}^{\mathbf{v}_{k+1}}$.
 \end{enumerate}
\end{lemma}
\begin{proof}
By Definition~\ref{responsibility}, the assumption of the lemma implies that agent $1$ is not responsible under any profile of the mechanism $M_k(\mathbf{s}_1,\dots,\mathbf{s}_k)$. Thus, the statement of the lemma follows from Lemma~\ref{AAA} and Definition~\ref{GDF definition}.
\end{proof}

In the previous sections, we have been able to give formulae {\sf df}, {\sf gf}, and {\sf rf} for the sets {\sf DF}, {\sf GF}, and {\sf RF} explicitly. Defining such a formula for set {\sf GDF} is a bit more involved. We do this through backward induction. Namely, let, for any $k$ such that $0\le k\le n$,
\begin{equation}\label{gdfi formula}
{\sf gdf}_k=
\begin{cases}
\gamma & \text{if $k=n$},\\ 
((\exists \mathbf{v}_{k+1}\forall\mathbf{v}_{k+2}\dots \mathbf{v}_n \gamma)\wedge \forall \mathbf{v}_{k+1}\,{\sf rf}_{k+1}) \\
\hspace{10mm}\vee (\neg(\exists \mathbf{v}_{k+1}\forall\mathbf{v}_{k+2}\dots \mathbf{v}_n \gamma)
\wedge\forall \mathbf{v}_{k+1}\,{\sf gdf}_{k+1}) &\text{if $k<n$.}
\end{cases}
\end{equation}
Note that formula ${\sf gdf}_k$ has free variables from sets $\mathbf{v}_1$,\dots,$\mathbf{v}_k$. In particular, formula ${\sf gdf}_0$ is a closed formula. In Claim~\ref{gdf0 claim}, later in this proof, we show that ${\sf gdf}_0$ gives the characterisation of set {\sf GDF} that we are looking for.

\begin{lemma}
Set {\sf GDF} belongs to class $\Pi_2$. 
\end{lemma}
\begin{proof}
\begin{claim}\label{12-aug-a}
${\sf gdf}_k$ is a $\Pi_2$-formula.    
\end{claim}
\begin{proof-of-claim}
We prove the claim by backward induction on $k$. If $k=n$, then ${\sf gdf}_k=\gamma\in \Pi_2$ because Boolean formula $\gamma$ is quantifier-free by Definition~\ref{mechanism definition}.  

Suppose that $k<n$. Observe that, by equation~\eqref{rf}, the formula 
$$
{\sf rf}_{k+1}\to (\forall\mathbf{v}_{k+2}\dots \mathbf{v}_n \gamma
\leftrightarrow
\exists\mathbf{v}_{k+2}\dots \mathbf{v}_n \gamma
)
$$
has value 1 for all possible valuations of free variables. Hence, the same is true about the formula
$$
\forall\mathbf{v}_{k+1}{\sf rf}_{k+1}\to (\forall\mathbf{v}_{k+2}\dots \mathbf{v}_n \gamma
\leftrightarrow
\exists\mathbf{v}_{k+2}\dots \mathbf{v}_n \gamma
).
$$
Then, the same is also true about the formula
$$
\forall\mathbf{v}_{k+1}{\sf rf}_{k+1}\to (\exists\mathbf{v}_{k+1}\forall\mathbf{v}_{k+2}\dots \mathbf{v}_n \gamma
\leftrightarrow
\exists\mathbf{v}_{k+1}\exists\mathbf{v}_{k+2}\dots \mathbf{v}_n \gamma
).
$$
Hence,
$$(\exists\mathbf{v}_{k+1}\forall\mathbf{v}_{k+2}\dots \mathbf{v}_n \gamma)\wedge \forall\mathbf{v}_{k+1}{\sf rf}_{k+1}
\equiv
(\exists\mathbf{v}_{k+2}\dots \mathbf{v}_n \gamma)\wedge \forall\mathbf{v}_{k+1}{\sf rf}_{k+1}.
$$ 
Thus, the formula 
$(\exists\mathbf{v}_{k+1}\forall\mathbf{v}_{k+2}\dots \mathbf{v}_n \gamma)\wedge \forall\mathbf{v}_{k+1}{\sf rf}_{k+1}$ belongs to class  $\Pi_2$ by equation~\eqref{rf}.
Note that the formula 
$\neg\exists \mathbf{v}_{k+1}\forall\mathbf{v}_{k+2}\dots \mathbf{v}_n \gamma$ is also a $\Pi_2$-formula. 
Finally, 
$\forall \mathbf{v}_{k+1}{\sf gdk}_{k+1}$ is a $\Pi_2$-formula because, by the induction hypothesis, ${\sf gdk}_{k+1}$ is a $\Pi_2$-formula.
Therefore, ${\sf gdf}_k$ is a $\Pi_2$-formula by equation~\eqref{gdfi formula}.
\end{proof-of-claim}

\begin{claim}
$M_k(\mathbf{s}_1,\dots,\mathbf{s}_k)\in  {\sf GDF}$ iff ${\sf gdf}_k[\mathbf{s}_1,\dots,\mathbf{s}_k/\mathbf{v}_1,\dots,\mathbf{v}_k]\equiv\T$.
\end{claim}
\begin{proof-of-claim}
We prove the statement of the claim by backward induction on $k$. 

In the base case, when $k=n$, the mechanism $M_k(\mathbf{s}_1,\dots,\mathbf{s}_k)\in  {\sf GDF}$ has no agents at all, as we have discussed after Definition~\ref{partial mechanism}. Thus, none of them is responsible. In this case, the statement
$M_n(\mathbf{s}_1,\dots,\mathbf{s}_n)\in  {\sf GDF}$ is equivalent  to 
$\gamma[\mathbf{s}_1,\dots,\mathbf{s}_n/\mathbf{v}_1,\dots,\mathbf{v}_n]\equiv\T$
by Definition~\ref{GDF definition}. The last statement, by equation~\eqref{gdfi formula}, is equivalent to the statement ${\sf gdf}_n[\mathbf{s}_1,\dots,\mathbf{s}_n/\mathbf{v}_1,\dots,\mathbf{v}_n]\equiv\T$.

In the induction case, where $k<n$, we consider the following two cases separately:

\vspace{1mm}\noindent
{\em Case 1}: $\exists \mathbf{v}_{k+1}\forall\mathbf{v}_{k+2}\dots \mathbf{v}_n (\gamma[\mathbf{s}_1,\dots,\mathbf{s}_k/\mathbf{v}_1,\dots,\mathbf{v}_k])\equiv\T$. Then, by Lemma~\ref{green layer lemma}, the statement 
$M_k(\mathbf{s}_1,\dots,\mathbf{s}_k)\in  {\sf GDF}$
is equivalent to the statement
$$M_{k+1}(\mathbf{s}_1,\dots,\mathbf{s}_{k+1})\in  {\sf RF}\text{ for each tuple }\mathbf{s}_{k+1}\in\{0,1\}^{|\mathbf{v}_{k+1}|}.$$
By Lemma~\ref{RF lemma}, the above statement is equivalent to 
$${\sf rf}_{k+1}[\mathbf{s}_1,\dots,\mathbf{s}_{k+1}/\mathbf{v}_1,\dots,\mathbf{v}_{k+1}]\equiv\T\text{ for each tuple }\mathbf{s}_{k+1}\in\{0,1\}^{|\mathbf{v}_{k+1}|}.$$
The previous statement is equivalent to the statement
$$(\forall\mathbf{v}_{k+1}{\sf rf}_{k+1})[\mathbf{s}_1,\dots,\mathbf{s}_{k}/\mathbf{v}_1,\dots,\mathbf{v}_{k}]\equiv\T.$$ By the assumption of the case and equation~\eqref{gdfi formula}, the last statement is equivalent to the statement
${\sf gdf}_k[\mathbf{s}_1,\dots,\mathbf{s}_k/\mathbf{v}_1,\dots,\mathbf{v}_k]\equiv\T$.

\vspace{1mm}\noindent
{\em Case 2}: $\exists \mathbf{v}_{k+1}\forall\mathbf{v}_{k+2}\dots \mathbf{v}_n (\gamma[\mathbf{s}_1,\dots,\mathbf{s}_k/\mathbf{v}_1,\dots,\mathbf{v}_k])\equiv\F$. Then, by Lemma~\ref{no green layer lemma}, the statement
$M_k(\mathbf{s}_1,\dots,\mathbf{s}_k)\in  {\sf GDF}$
is equivalent to the statement
$$M_{k+1}(\mathbf{s}_1,\dots,\mathbf{s}_{k+1})\in  {\sf GDF}\text{ for each tuple }\mathbf{s}_{k+1}\in\{0,1\}^{|\mathbf{v}_{k+1}|}.$$
The above statement, by the induction hypothesis, is equivalent to 
$${\sf gdf}_{k+1}[\mathbf{s}_1,\dots,\mathbf{s}_{k+1}/\mathbf{v}_1,\dots,\mathbf{v}_{k+1}]\equiv\T\text{ for each tuple }\mathbf{s}_{k+1}\in\{0,1\}^{|\mathbf{v}_{k+1}|}.$$
The previous statement is equivalent to 
$$(\forall\mathbf{v}_{k+1}{\sf gdf}_{k+1})[\mathbf{s}_1,\dots,\mathbf{s}_{k+1}/\mathbf{v}_1,\dots,\mathbf{v}_{k+1}]\equiv\T.$$
By the assumption of the case and equation~\eqref{gdfi formula}, the last statement is equivalent to the statement
${\sf gdf}_k[\mathbf{s}_1,\dots,\mathbf{s}_k/\mathbf{v}_1,\dots,\mathbf{v}_k]\equiv\T$.
\end{proof-of-claim}

The next statement follows from the above claim and Definition~\ref{partial mechanism}.
\begin{claim}\label{gdf0 claim}
$(n,\mathbf{v},\gamma)\in {\sf GDF}$ iff  ${\sf gdf}_0\equiv\T$.
\end{claim}
The statement of the lemma follows from Claim~\ref{12-aug-a} and Claim~\ref{gdf0 claim}.
\end{proof}

\begin{lemma}
Set {\sf GDF} is $\Pi_2$-hard. 
\end{lemma}
\begin{proof}
Consider any Boolean formula $\phi$ without quantifiers whose variables are divided into two disjoint sets, $\mathbf{x}$ and $\mathbf{y}$. Define mechanism $(n,\mathbf{v},\gamma)$ as follows:
\begin{enumerate}
    \item $n=2$,
    \item $\mathbf{v}_1=\mathbf{x}$, $\mathbf{v}_2=\mathbf{y}\cup\{z\}$, where $z$ is any Boolean variable such that $z\notin \mathbf{x}\cup \mathbf{y}$,
    \item $\gamma=\phi\wedge z$.
\end{enumerate}
Note that $\forall \mathbf{x} \exists \mathbf{y} \phi$ is a closed formula. Thus, it is semantically equivalent to either $\F$ or $\T$.
To prove the statement of the lemma, it suffices to establish the following claim.

\begin{claim}
$\forall \mathbf{x} \exists \mathbf{y}\phi\equiv\T$ iff $(n,\mathbf{v},\gamma)\in {\sf GDF}$.
\end{claim}
\begin{proof-of-claim}
Note that the deontic constraint $\gamma=\phi\wedge z$ has value $0$ if Boolean variable $z$ has value $0$. Thus, the first agent does not have a way to guarantee that the deontic constraint is satisfied. Hence, the first agent is not responsible under any action profile. Then, by Definition~\ref{responsibility} and Definition~\ref{GDF definition}, it suffices to show that $\forall \mathbf{x} \exists \mathbf{y}\phi\equiv\T$ iff the second agent is responsible under each action profile that does not satisfy deontic constraint $\gamma$.

\noindent
$(\Rightarrow):$  The second agent can guarantee that formula $\phi$ has value 1 by the assumption $\forall \mathbf{x} \exists \mathbf{y}\phi\equiv\T$. The second agent can also guarantee that $z$ has a value 1. Thus, the second agent has a strategy that guarantees that the deontic constraint $\gamma=\phi\wedge z$ is satisfied. Therefore, the second agent is responsible under each action profile that does not satisfy deontic constraint $\gamma$.

\noindent
$(\Leftarrow):$ Suppose that $\forall \mathbf{x} \exists \mathbf{y}\phi\equiv\F$. Thus, there is a tuple $\mathbf{x}_0$ such that $(\forall \mathbf{y}\neg\phi)[\mathbf{x}_0/\mathbf{x}]\equiv\T$. Hence, if the first agent chooses action $\mathbf{x}_0$, then the deontic constraint $\gamma=\phi\wedge z$ is not satisfied no matter what the second agent does. Therefore, if the first agent chooses action $\mathbf{x}_0$, then the constraint is not satisfied and, by Definition~\ref{responsibility}, the second agent is not responsible for this.~
\end{proof-of-claim}
This concludes the proof of the lemma.
\end{proof}

The next theorem follows from the two lemmas above.
\begin{theorem}
The set ${\sf GDF}$ is $\Pi_2$-complete.    
\end{theorem}

\section{Conclusion}\label{conclusion section}

Responsibility gap and diffusion of responsibility are two usually undesirable phenomena in decision-making mechanisms. Sometimes the responsibility gap can be eliminated by a mechanism that requires agents to provide their input in order, rather than acting simultaneously. However, introducing an order might lead to the diffusion of responsibility even if it is not present in the simultaneous setting. In this article, we studied the computational complexity of detecting responsibility gaps and diffusion of responsibility in a setting where all agents act in a fixed order. The main technical results are that the sets of diffusion-free and gap-free mechanisms are, respectively, $\Pi_2$- and $\Pi_3$-complete and that the intersection of these two sets is $\Pi_2$-complete.

These results show that gap-free, diffusion-free, and gap-and-diffusion-free mechanisms not only exist in abundance, but also that the structure of these three classes of mechanisms is non-trivial. At the same time, these results prove that designing responsible mechanisms cannot be done by a brute-force method or simply by luck of choosing the right one and verifying that it works. Instead, the design of decision-making mechanisms that enforce individual accountability is an endeavor that requires ingenuity.

\bibliographystyle{elsarticle-num}

\bibliography{naumov}

\end{document}